\documentclass[letterpaper, 10 pt, conference]{ieeeconf}  % Comment this line out if you need a4paper

\IEEEoverridecommandlockouts                              % This command is only needed if 
\usepackage{url}
\usepackage{times}
\usepackage{multicol}
\usepackage[bookmarks=true]{hyperref}
\usepackage{graphics} % for pdf, bitmapped graphics files
\usepackage{epsfig} % for postscript graphics files
\usepackage{mathptmx} % assumes new font selection scheme installed
\usepackage{times} % assumes new font selection scheme installed
\usepackage{amsmath,amsfonts,bm}
\usepackage{amssymb}  % assumes amsmath package installed
\usepackage{mdframed}
\usepackage{mathtools}
\usepackage{booktabs}
\usepackage[mathcal]{euscript}
\usepackage{esint}
\usepackage{amssymb}
\usepackage{tabularray}
\usepackage{cancel}
\usepackage{tikz}
\usepackage{rotating}
\usepackage{stackengine}

\tikzset{
	MyPersp/.style={scale=1.8,x={(-0.8cm,-0.4cm)},y={(0.8cm,-0.4cm)},
    z={(0cm,1cm)}},
	MyPoints/.style={fill=white,draw=black,thick}
		}

\newtheorem{theorem}{Theorem}
\newtheorem{corollary}{Corollary}[theorem]
\newtheorem{lemma}[theorem]{Lemma}
\newtheorem{remark}{Remark}

\newcommand\IncG[2][]{\addstackgap{%
  \raisebox{-.5\height}{\includegraphics[#1]{#2}}}}

\newcommand{\ambient}{\mathbb{R}^{3}}
\newcommand{\ambientvec}{\mathbb{R}^3}

\newcommand{\confspace}{\mathcal{C}}
\newcommand{\Rplus}{\mathbb{R}^+}
\newcommand{\funspace}{\mathcal{F}}

\newcommand{\gammareg}{A[\gamma]}
\newcommand{\gammavol}{V[\gamma]}

\newcommand{\tanb}{\operatorname{T}}
\newcommand{\qdot}{\dot{q}}
\newcommand{\gt}{\gamma_t}
\newcommand{\gtdot}{\dot{\gamma}_t}

\newcommand{\RT}{\mathbb{R}^3}

\newcommand{\MSVDI}{\textit{minimum swept volume distance} }

\newcommand{\MSVD}{minimum swept volume distance }
\newcommand{\MSV}{minimum swept volume }

\def\comment#1{{}}

\title{\LARGE \bf
A minimum swept-volume metric structure for configuration space
}

\author{
Yann de Mont-Marin$^{\text{*,a,b}}$, Jean Ponce$^{\text{b,c}}$ and Jean-Paul Laumond$^{\text{a,b}}$%
\thanks{$^{\text{*}}$Corresponding author}% <-this % stops a space
\thanks{$^{\text{a}}$Inria}% <-this % stops a space
\thanks{$^{\text{b}}$Département d'informatique de l'Ecole normale supérieure (ENS-PSL, CNRS, Inria)}%
\thanks{$^{\text{c}}$Center for Data Science, New York University}% <-this % stops a space
}

\begin{document}
\maketitle
\thispagestyle{empty}
\pagestyle{empty}

\begin{abstract}
Borrowing elementary ideas from solid mechanics and differential
geometry, this presentation shows that the volume swept by a regular
solid undergoing a wide class of volume-preserving deformations induces
a rather natural metric structure with well-defined and computable
geodesics on its configuration space. This general result applies to
concrete classes of articulated objects such as robot manipulators, and
we demonstrate as a proof of concept the computation of geodesic paths
for a free flying rod and planar robotic arms as well as their use in path planning with many obstacles.

% \keywords{configuration space, distance, swept volume, motion planning}
\end{abstract}

\section{Introduction}
\subsection{Context}
The concept of configuration space, first formalized in the robotics context by Lozano-Pérez \cite{LozanoPerez1983}, is a key element for any approach to motion planning \cite{Latombe91, Lav06}. In turn, developing effective algorithms for this task requires equipping the configuration space with an adequate metric structure, which is a key factor in developing the sampling and interpolation strategies that often are at the core of these algorithms.

Lavalle underlines in \cite{Lav06} the importance of the choice of distance and argues that it must capture within the same quantity different degrees of freedom of the system which have different units (for example an angle and a length) and that this issue can be fixed by considering a physically meaningful distance.
In this paper we postulate that the minimal volume swept by a robot between two configurations is a good candidate in presence of obstacles.
This follows Kuffner's suggestion in \cite{Kuffner2004} in the case of the motion of a solid. He writes:
\textit{Intuitively, an ideal metric for path planning in $SE(3)$ would correspond to a measure of the minimum swept volume [...]. Intuitively, minimizing the swept volume will minimize the chance of collision with  obstacles.}
%, which in turn maximizes the chance of discovering collision free paths between pairs of configurations

We propose in this paper to formalize and extend this intuition to configuration spaces for regular bodies undergoing a wide class of volume-preserving deformations using tools from solid mechanics and differential geometry. We introduce the notion of \textit{minimum swept volume distance\footnote{Here the term swept volume refers to the measure of the swept region and not to the region itself. In addition, the notion of \textit{minimum swept volume distance} is new and is not to be confused with the notion of \textit{swept volume distance}}}
%, see \ref{sec:relatedwork}}
 as the minimum amount over all possible paths of volume to swept from one configuration to another.
First in a very broad set-theoretical setting and then in a differential setting where the configuration space is identified to a set of diffeomorphisms resembling the construction of Lin and Burdwick in the case of rigid bodies \cite{Lin97}. Under some mild assumptions, we prove that the \MSV is indeed a distance with well-defined and computable geodesics.
This is intuitively rather natural, but we are not aware of any previous statements of this result.

In addition to this theoretical contribution, our construction is an important step toward making the use of swept volumes in trajectory optimization \cite{Jallet21, Schulman2014} computationally realistic. Indeed, traditionally, one first computes the region swept along a given trajectory, which is very costly and then measures the volume of the region. This prevents the optimization of the trajectory with respect to the swept volume. In comparison, our construction approaches the swept volume as a sum of local contributions. This allows us to formulate the problem of computing the \MSV as an energy minimization problem and to use a state of the art second-order optimization solver \cite{Wright97}.

In our experiments, we compute some geodesics for different solids and for some planar robotic arms. We finally use the \MSVD and its geodesics for interpolation in the well-known motion planning algorithm RRT \cite{Lav06} and show with multiple challenging examples a significant reduction of the number of nodes explored.

The main contributions of the paper are:
\begin{itemize}
\item the construction of a \textit{local swept volume} operator,
\item proofs that the \MSV is a distance on the configuration space in a set-theoretical setting and in a  differential setting using the local operator,
\item a closed form for the operator in the case of a poly articulated system with rigid bodies, and
\item the numerical computation of the minimum swept volume distance, its geodesics, and their use in RRT.
\end{itemize}
\iffalse
\begin{figure*}[t!]
\vspace{-4em}
  \centerline{%
  \includegraphics[width=1.2\textwidth]{it_is_it.pdf}
  }
    \vspace{-.5em}
  % Illustration of the swept volume of different motion between two configurations of a solid. \textbf{Middle:}
  \caption{\footnotesize \textbf{Left:} Set-theoretical definition of the swept volume. A region can be swept twice.
  %A body deforms along the path $\gamma$ from $B_0$ to $B_1$ with constant volume $U$. The region $A[\gamma]$ swept along $\gamma$ is the union of $B_0$, $B_1$ and the XXX region in    between $C$. Its swept volume is $V[\gamma] =V(A[\gamma])-U=
    % V(C\cup B_0)=V(C\cup B_1)$.
    \textbf{Right:} Differential definition of the \MSV}
% does not surely possess geodesic because a region in XXX can be swept twice in different sub paths.
  \label{fig:set_illu}
\vspace{-.5em}
\end{figure*}
\fi
\subsection{Related work}\label{sec:relatedwork}
\textbf{Swept volume.}
\iffalse
The term \textit{swept volume} is used in the robotics literature both to designate the region of space swept by a moving system \cite{Kim2004, Peternell2005, Sellan2021, Yu2013} and the scalar volume of this region \cite{Chiang2020}.
In this paper, unless otherwise stated, swept volume refers to the scalar quantity.
It should also be noted that the notion of \MSVD introduced in this paper is not to be confused with \textit{swept volume distance} of \cite{Taubig2011, Xav97} which is not a distance on the configuration space but designates the smallest Euclidean distance in the workspace between the points of the swept region and the points of the obstacles of the environment.
\fi
Several algorithms are available to compute the region of space, swept by a moving robot, using voxels \cite{Peternell2005, Yu2013} or its boundary using meshes \cite{Kim2004, Sellan2021}. Despite algorithmic progress both computations remain extremely expensive.
Blackmore and Leu discuss in \cite{Blackmore1992} an analytical formulation of the swept region that may eventually lead to an efficient swept volume formulation but since then, this approach does not appear to have been explored further.
Recent works such as \cite{Baxter2020, Markvorsen2016, Orthey2019} build heuristics around the swept volume to weight a Euclidean distance, unfortunately this approach tacitly assumes that the swept volume distance have a Euclidean structure, which is limiting.
More recently Chiang first computes the swept volume of some trajectories and uses neural networks to train an estimator of the swept volume \cite{Chiang2020}. Although this partially addresses the problem of computational cost, neither the quantity used, nor its estimate constitute a distance on the configuration space.

\textbf{Configuration space metric structure.}
Zhang et al. study in \cite{Zhang2007} the maximum vertex distance originally proposed by Lavalle in \cite{Lav06}. Although this distance is derived from a physical quantity and thus allows the homogenization of the different degrees of freedom, it may not be adapted to the problem of exploration in a constrained environment since the interaction of the system with the environment will depend on the displacement of all the points and not only the largest displacement. Zefran et al. extensively study Riemannian geometries on the configuration space of a solid and establish the properties necessary for a geometric structure to respect the fundamental principles of mechanics, but this study is restricted to the case of a configuration space isomorphic to $SE(3)$ and does not cover the swept volume case.

\section{Swept-volume distance construction}\label{sec:constru}
\begin{figure}[t!]
  \centerline{%
  \includegraphics[width=.45\textwidth]{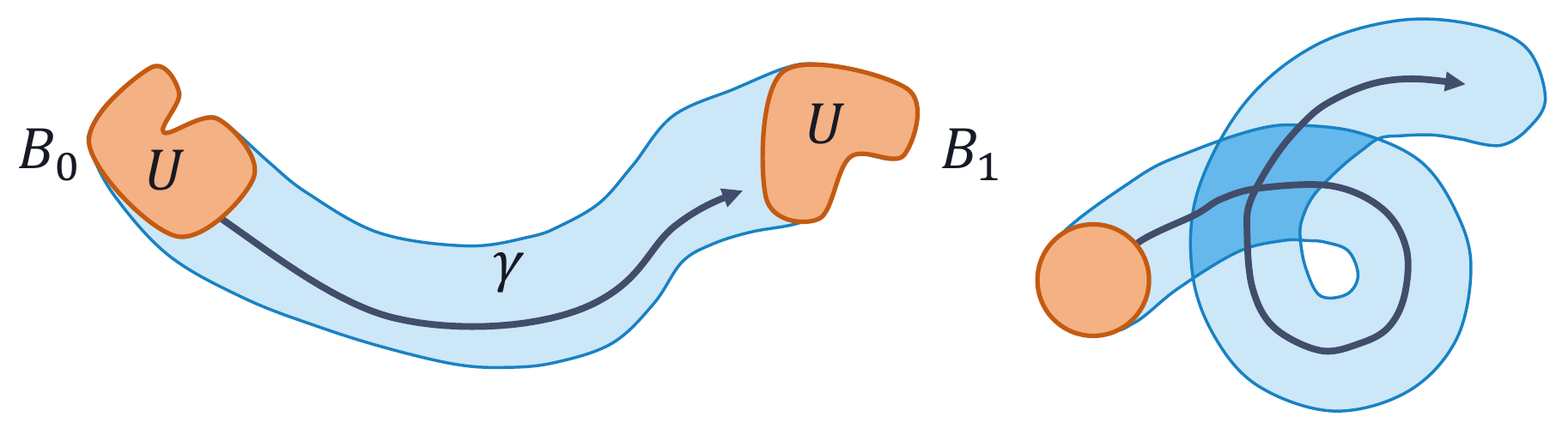}
  }
  \caption{\footnotesize \textbf{Left:} Set-theoretical notations for the volume swept by a body along a path $\gamma$. \textbf{Right:} Illustration of a region swept twice and counted one.}
\label{fig:set_setting}
\vspace{-.5em}
\end{figure}

\subsection{A set theoretical formulation}\label{sec:set}
Let us equip the set of all compacts of the affine Euclidean space $\mathbb{E}^3$ with the topology induced by the Hausdorff distance and define a regular body as a compact equal to the closure of its interior \cite{Milnor1965}. For the sake of simplicity, we identify $\mathbb{E}^3$ as well as the vector space of its translations with $\RT$, in the rest of the paper. Although this implies the choice of a coordinate system, all our construction is independent of that choice.

In this section we call configuration space, denote by $\confspace$, a set of regular bodies with constant volume $U > 0$. We require $\confspace$ to be path connected under the Hausdorff topology. Note that the largest possible configuration space formed by the set containing all the regular bodies of volume $U$ is itself path connected. Given a continuous path $\gamma : [0,1] \rightarrow \confspace$ we define the region swept along the motion $\gamma$ and the corresponding swept volume by:
\begin{align}
    \gammareg &= \bigcup_{t\in[0,1]}\gamma(t) \text{ and }
    \gammavol = V(\gammareg) - U,
\end{align}
where $V$ denotes the usual volume. Figure \ref{fig:set_setting} (left) illustrates the notations.
%\footnote{The volume is defined with the natural Lebesgue measure on $\ambient$ since compact set are always measurable.}
\begin{lemma}\label{thm:dist_set}
Given two configuration $B_0$ and $B_1$ in $\mathcal{C}$, let us denote by $\Gamma(B_0, B_1)$ the set of all continuous paths joining $B_0$ and $B_1$ in $\confspace$. The mapping $d : \confspace^2 \rightarrow \Rplus$ defined by
\begin{equation}
    d(B_0, B_1) = \inf_{\gamma \in \Gamma(B_0, B_1)}\gammavol\label{eq:dist1}
\end{equation}
is a distance on $\mathcal{C}$ and we call it the \MSVDI between $B_0$ and $B_1$.
\end{lemma}
\begin{proof}
  %First 
  $d$ is well defined on $\confspace^2$, indeed $\Gamma(B_0, B_1)$ is never empty because $\mathcal{C}$ is path connected.
  Given any configurations $B_0$ and $B_1$, we have $d(B_0,B_1)\ge 0$
  since for any $\gamma$ in $\Gamma(B_0,B_1)$, $B_0$ (and of course
  $B_1$) is a subset of $\gammareg$ with volume $U$.
  The function $d$
  is symmetric by definition (any path joining $B_0$ and $B_1$ also
  joins $B_1$ and $B_0$), and it satisfies the triangular inequality
  because given any path $\gamma_{01}$ joining $B_0$ to $B_1$ and any
  path $\gamma_{12}$ joining $B_1$ and $B_2$, the path $\gamma_{02}$
  obtained by concatenating $\gamma_{01}$ and $\gamma_{12}$ in $B_1$
  joins $B_0$ and $B_2$, with
  $V[\gamma_{02}] =
  %V(A[\gamma_{13}]) - U = 
  V(A[\gamma_{01}]\cup A[\gamma_{12}]) - U
%  V(A[\gamma_{12}])+V(A[\gamma_{23}])- V(A[\gamma_{12}]\cap A[\gamma_{23}])- U
  \le V(A[\gamma_{01}])+V(A[\gamma_{12}])-2U \le V[\gamma_{01}]+V[\gamma_{12}]$ since
  $B_2\subset A[\gamma_{12}]\cap A[\gamma_{23}]$ with volume $U$.  Now suppose
   $B_1\neq B_0$. We can always find an open ball which is in $B_0 \backslash B_1$ (by swapping $B_0$ and $B_1$ if needed), so we have: $V(B_0\cup B_1) = V(B_1) + V(B_0 \backslash B_1) > U$. For any path $\gamma$ joining
  $B_0$ to $B_1$ we have
  $\gammavol=V(\gammareg) - U\ge V(B_0\cup B_1) - U > 0$.
  %and the infimum is strictly positive because $V(B_0\cup B_1) - U>0$ does not depend on $\gamma$
  Trivially, we have $d(B_0,B_0) = 0$ because a path $\gamma_0$ such that
  $\gamma_0(t) = B_0$ for any $t$ in an open interval verify $V[\gamma_0]= V(B_0) - U=0$.
\end{proof}
The \MSVD $d$ is generally not a geodesic distance. The length of a geodesic must be the sum of the lengths of its two half sub-paths, but the swept volume $V$ count a region visited $n$ times only once.
%as illustrated in Figure \ref{fig:set_illu} (left). 
Then for a candidate geodesic the sum of the swept volumes for two sub-paths visiting the same region may not be equal to the swept volume of the complete path. We illustrate the phenomenon in  Figure \ref{fig:set_setting} (right)

In robotics we need to be able to build geodesics to interpolate between configurations. To obtain a distance with geodesics we now use a differential setting and construct the distance as an integral of local contributions.

\subsection{A differential formulation}\label{sec:diff}
\begin{figure}[t!]
  \centerline{%
  \includegraphics[width=.5\textwidth]{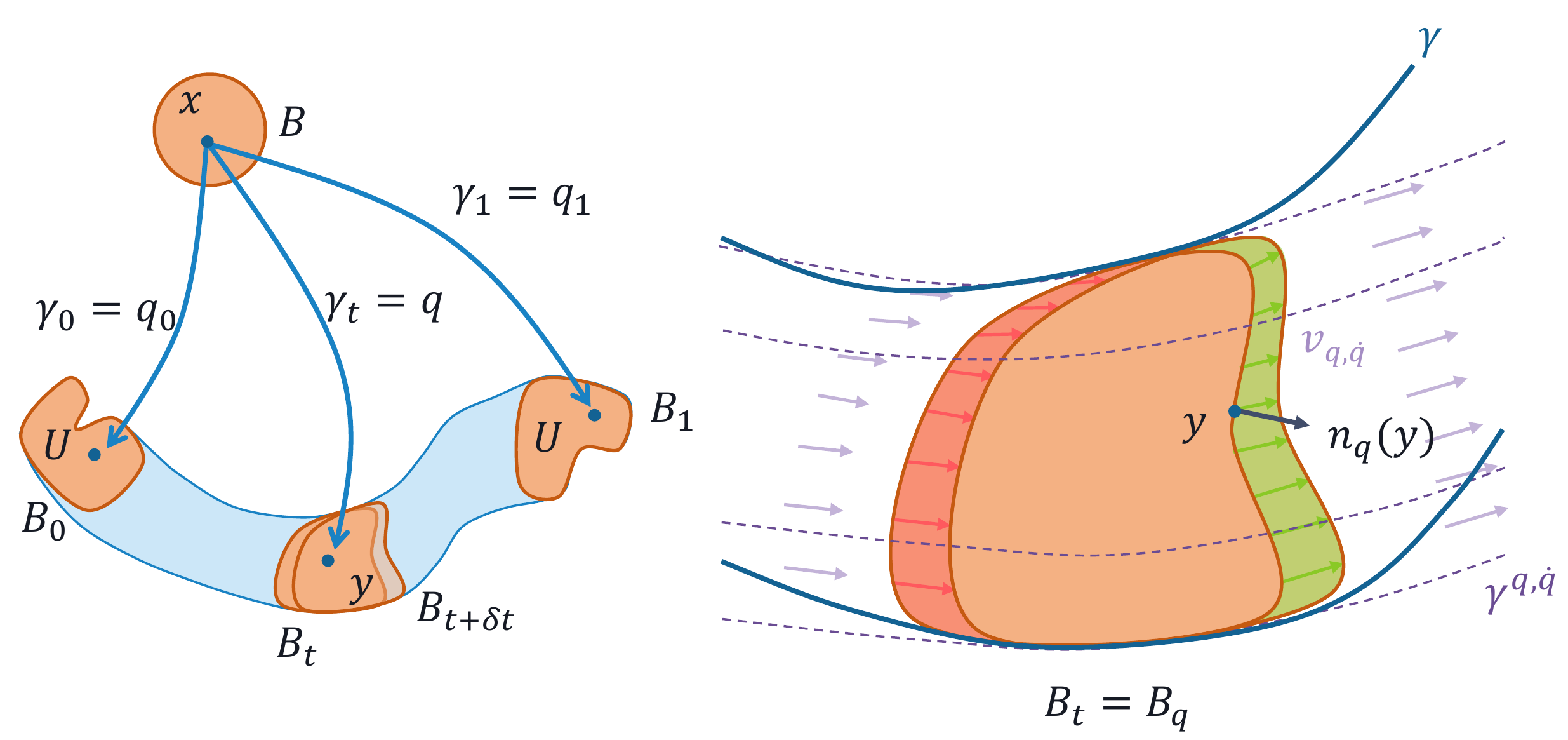}
  }
  \caption{\footnotesize \textbf{Left:} Notations of the differential point of view  of the volume swept by a body along a path. \textbf{Right:} Highlight of the local contribution at time  $t$, in green the outgoing flux, in red the incoming flux. We also represent in purple the flow of the static velocity field $v_{q,\qdot}$.}
\label{fig:diff_setting}
\vspace{-.5em}
\end{figure}
\iffalse
\begin{figure*}[htb]
  \centerline{%
    \includegraphics[width=0.3\textwidth]{diffsweep.pdf}
    \hspace{3mm}
    \includegraphics[width=0.3\textwidth]{diffsweepclose.pdf}}
  \caption{\footnotesize \textbf{Left:} Differential setting and swept
    volume. \textbf{Middle:} The geometric elements defining the local incremental flux. \textbf{Right:} In XXX $\gamma^{q,\dot{q}}$ the unique path extension with static velocity field around a configuration $q$ with tangent $\qdot$ and in XXX a standard path $\gamma$ passing through $q$ with velocity $\qdot$ at a time $t$}
  \label{fig:diff_illu}
\end{figure*}
\fi
Using a slightly different point of view, borrowed from solid mechanics \cite{Arnold1989}, we now consider a regular solid of reference $B$  in $\RT$ whose boundary $\partial B$ is a smooth orientable surface without boundary.
We consider a space $\funspace$ of diffeomorphisms of $\ambient$ to allow for deformations of the reference body $B$.
We require that $\funspace$ be a path connected differentiable manifold with respect to the strong functional topology, and that its elements be isochoric, i.e., preserve volume and orientation in $\RT$ \cite{Arnold1989}. Thus, for all $q$ in $\funspace$, $q$ is a diffeomorphism of $\ambient$ and $q(B)$ is a regular body with the same volume as $B$. We denote by $\mathcal{E}$ the mapping from $\funspace$ to $\confspace$ defined by $q\mapsto B_q$ and we make a first hypothesis \textbf{(H1)}: \textit{The embedding function $\mathcal{E}$ is injective.} With \textbf{(H1)}, $\mathcal{E}$ is a bijection between $\funspace$ and $\mathcal{C} = \mathcal{E}(\funspace)$ so we can identify the diffeomorphism $q$ and the regular body $B_q=\mathcal{E}(q)=q(B)$.

A differentiable path $\gamma$ in $\funspace$ is an $\funspace$-isotopy of the ambient space, i.e., a smooth map $I\times \ambient \rightarrow \ambient$ where $I$ is some interval of the real line such that for any time $t$, the mapping $\gamma_t: x\mapsto \gamma(t,x)$ is an element of $\funspace$. We can define its time derivative $\dot{\gamma}_t:\ambient \rightarrow \ambientvec$ for $x$ in $\ambient$ as $\dot{\gamma}_t(x) = \frac{\partial}{\partial t} \gamma(t,x)$.
%\begin{equation}
%\dot{\gamma}_t(x) = \frac{\partial}{\partial t} %\gamma(t,x).
%\end{equation}
For each $t$, the mapping $\dot{\gamma}_t$ belongs to the tangent space $\tanb_{\gamma_t}\funspace$.
Regardless of any path, a configuration $q$ in $\funspace$ and an element $\qdot$ in $\tanb_{q}\funspace$ induces a velocity field on $\ambient$ defined as $v_{q, \qdot}(y) = \qdot \circ q^{-1}(y)$.
%\begin{equation}
%v_{q, \qdot}(y) = \qdot \circ q^{-1}(y).
%\end{equation}
Note that we use the variable name $x$ to describe a point of the reference body $B$ in $\RT$ before the transformation and $y=q(x)$ for a point of $B_q$. We illustrate the notation in Figure \ref{fig:diff_setting} (left). Note also that since the elements of $\funspace$ are isochoric diffeomorphisms, the vector field $v_{q,\qdot}$ has zero divergence\footnote{It is a direct consequence of the transport theorem.} \cite{Arnold1989}. We can now define the local swept-volume as the flux of the velocity vector field through the boundary $\partial B_q$ of $B_q$:
\begin{equation}
F(q, \qdot) = \int_{\partial B_q} \operatorname{max}(0, <v_{q,\qdot}(y),n_{q}(y)>) dA(y),\label{eq:finsler}
\end{equation}
where $n_{q}(y)$ is the unit normal vector pointing outward at point $y$ of the boundary $\partial B_q$. Note that taking the max ensures we only consider the outgoing flux since incoming flux does not contribute to the
swept volume as illustrated in Figure \ref{fig:diff_setting} (right).
For a path $\gamma$ the additional volume swept by $B_{\gt}$ between $t$ and $t + \delta t$ is $F(\gt,\gtdot)\delta t$.
We can now define the flux-based swept volume associated with $\gamma$ as the sum of all local contribution:
\begin{equation}
    V'[\gamma] = \int_0^1 F(\gamma_t, \dot{ \gamma}_t) dt.\label{eq:vol_prime}
\end{equation}
Note that unlike the ordinary swept volume $V[\gamma]$, which counts the contribution of a region of space visited $n$ times only once, $V'[\gamma]$ counts it $n$ times. When no region is visited more than once,  $V'[\gamma] = V[\gamma]$.

\begin{figure*}[htp]
\centering
\setlength\tabcolsep{1.5pt} % default value: 6pt
\renewcommand{\arraystretch}{0.2} % General space between rows (1 standard)
\begin{tabular}{p{.15\textwidth}p{.15\textwidth}p{.15\textwidth}p{.15\textwidth}p{.15\textwidth}p{.15\textwidth}}
\IncG[width=.15\textwidth,height=.15\textwidth]{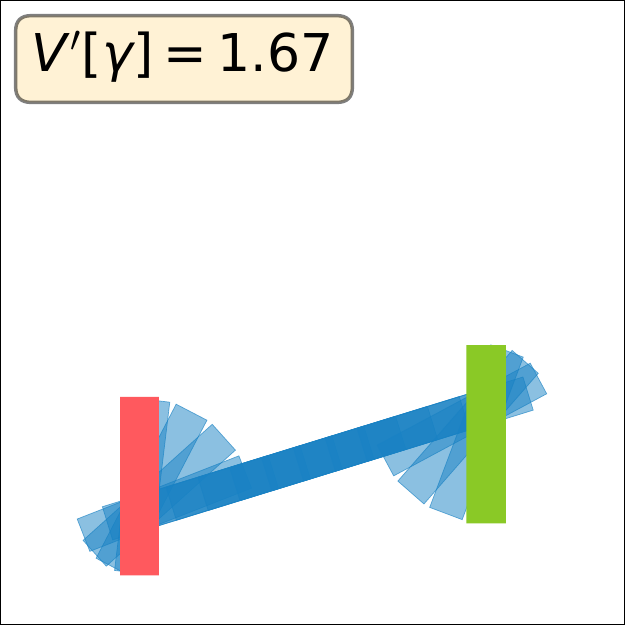}
&\IncG[width=.15\textwidth,height=.15\textwidth]{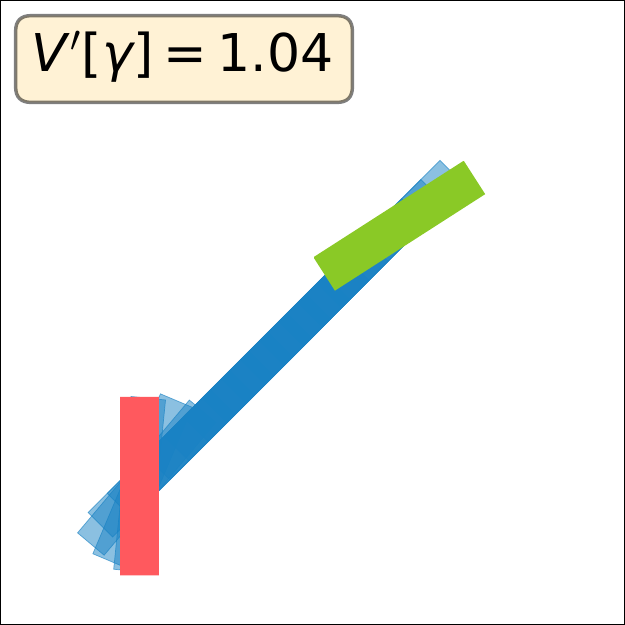}
&\IncG[width=.15\textwidth,height=.15\textwidth]{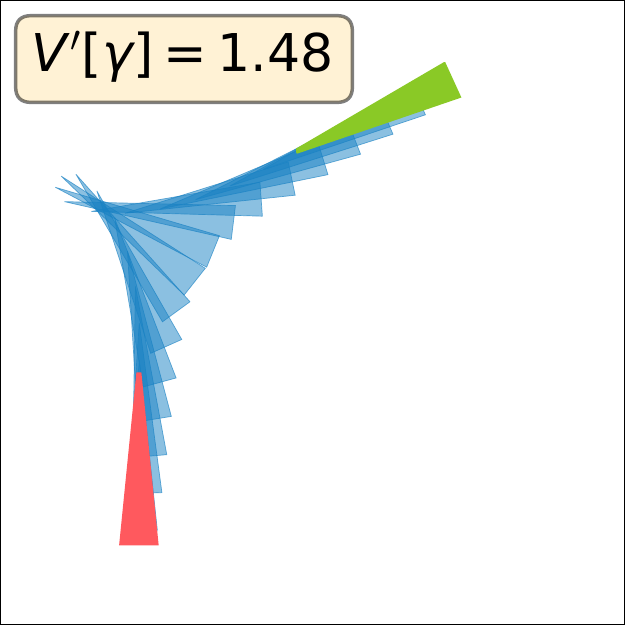}
&\IncG[width=.15\textwidth,height=.15\textwidth]{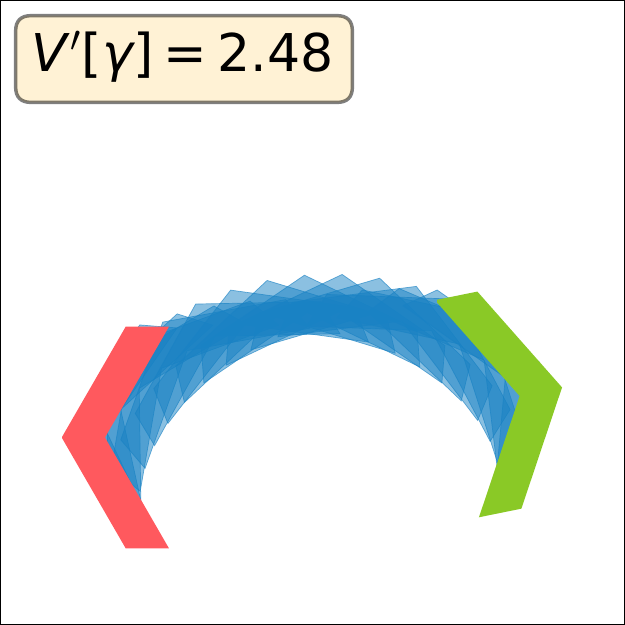}
&\IncG[width=.15\textwidth,height=.15\textwidth]{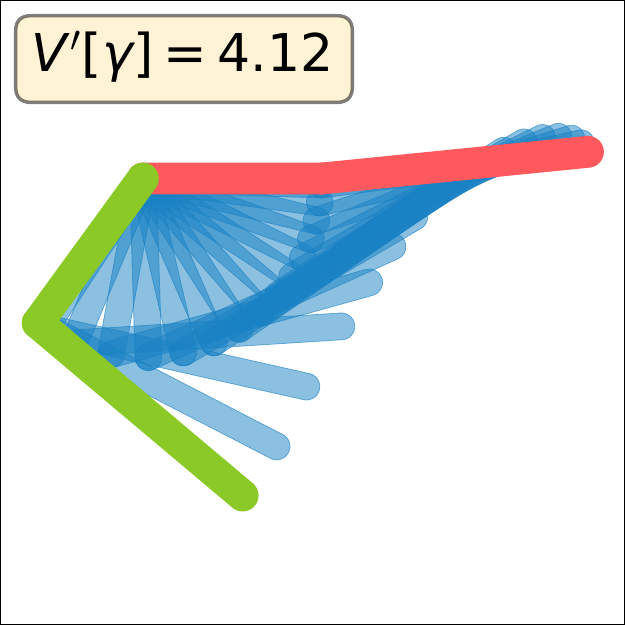}
&\IncG[width=.15\textwidth,height=.15\textwidth]{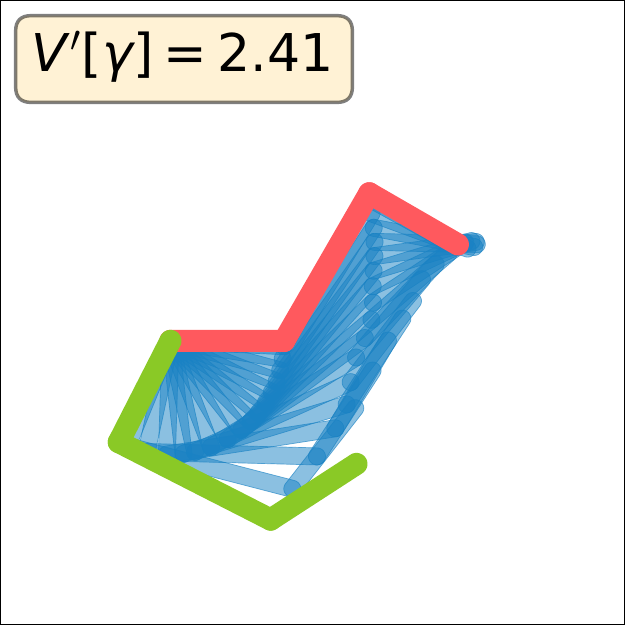}\\
\IncG[width=.15\textwidth,height=.15\textwidth]{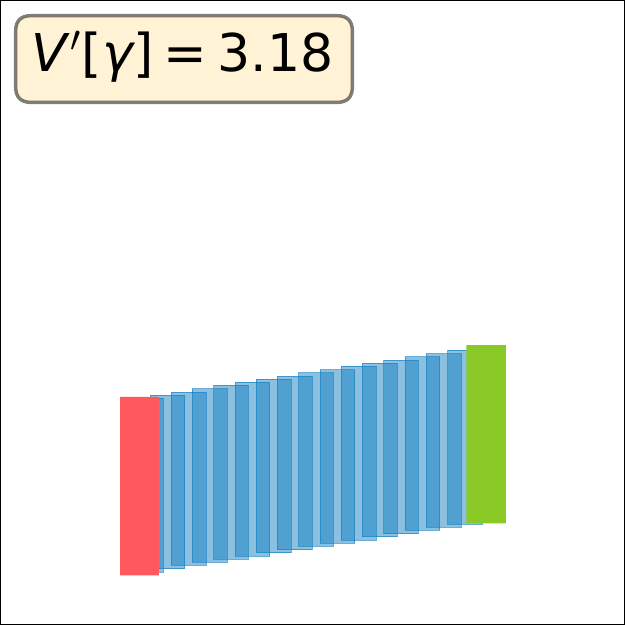}
&\IncG[width=.15\textwidth,height=.15\textwidth]{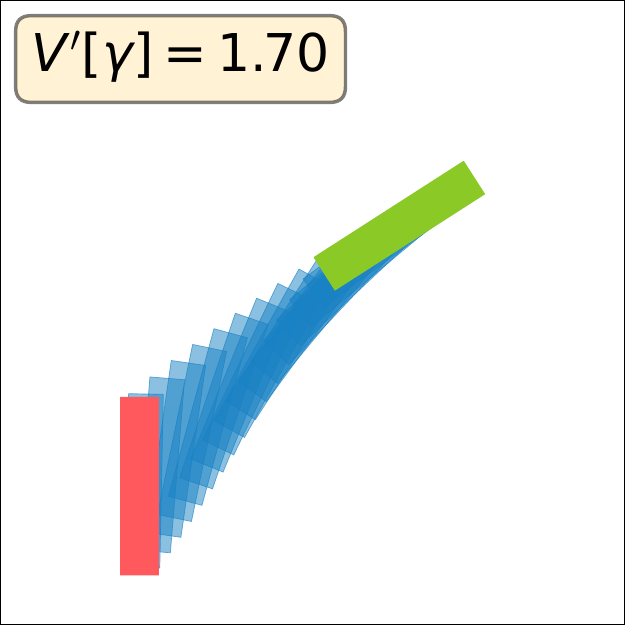}
&\IncG[width=.15\textwidth,height=.15\textwidth]{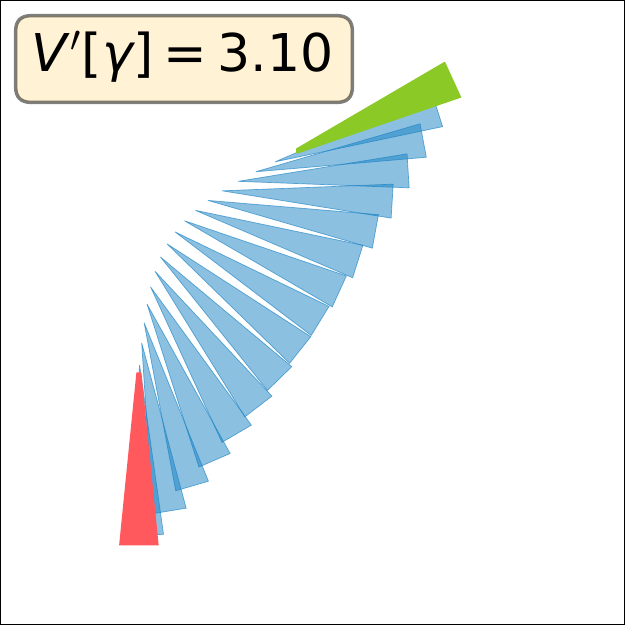}
&\IncG[width=.15\textwidth,height=.15\textwidth]{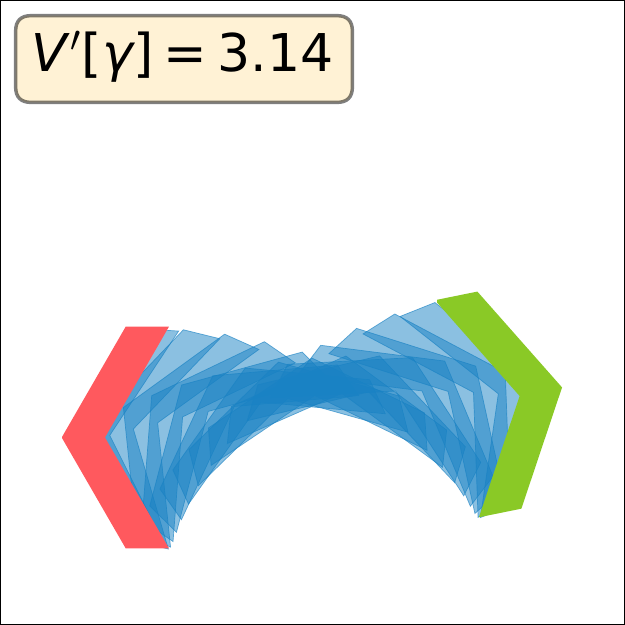}
&\IncG[width=.15\textwidth,height=.15\textwidth]{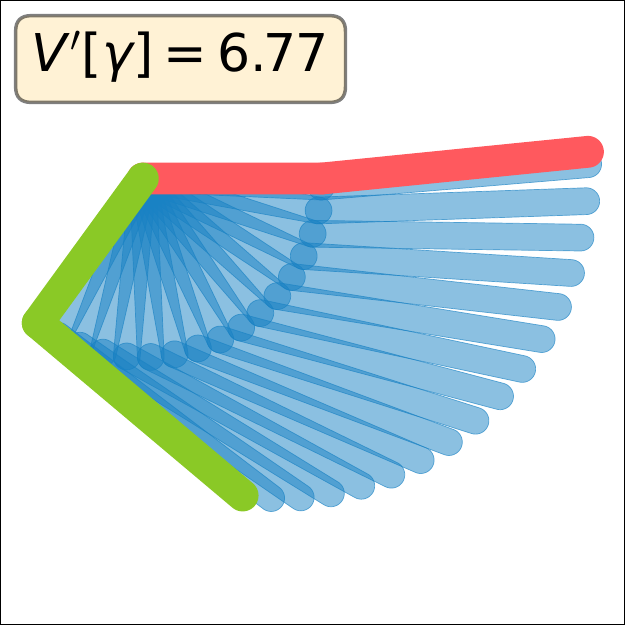}
&\IncG[width=.15\textwidth,height=.15\textwidth]{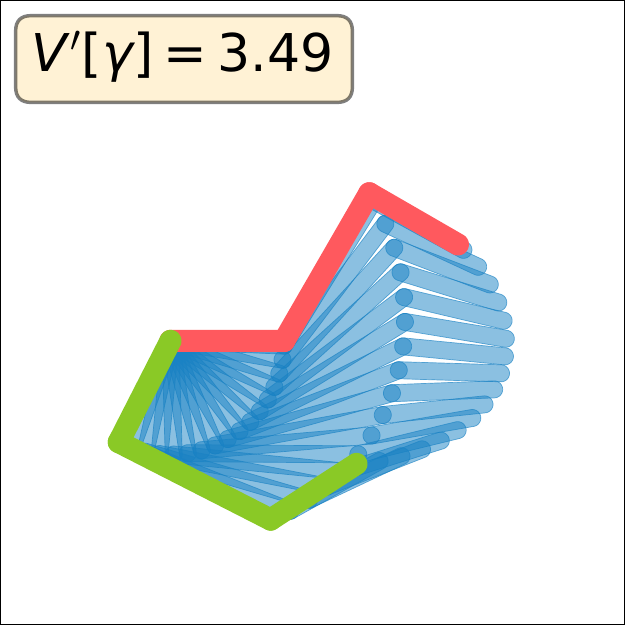}\\
\end{tabular}
  \caption{\footnotesize Geodesics for different systems and different pairs of configurations using the \MSV metric structure (top row) and the canonical Riemannian one (bottom row). From left to right: a rod (x2), an arrow-head shape, a bracket shape, a double pendulum, and a triple pendulum.}
  \label{fig:geodesic}
\vspace{-.5em}
\end{figure*}

In Eq. \eqref{eq:vol_prime}, $F$ can be interpreted as a Lagrangian on $\funspace$ and the flux-base swept volume $V'[\gamma]$ as the action of $F$ along $\gamma$. We can define the least action of $F$ between any pair of configurations $B_0$ and $B_1$ as
\begin{equation}
    d'(B_0, B_1) = \inf_{\gamma\in \Gamma^1(B_0, B_1)} V'[\gamma]\label{eq:dist_prime},
\end{equation}
where $\Gamma^1(B_0, B_1)$ is the set of differentiable paths in $\funspace$ such that $\mathcal{E}(\gamma_0)=\gamma_0(B)=B_0$ and $\mathcal{E}(\gamma_1)=\gamma_1(B)=B_1$.

In the remaining part of this subsection, we prove that $d'$ as in \eqref{eq:dist_prime} defines a distance on and we discuss the existence of its geodesics.
The key point to prove is that for all $q$ in $\funspace$ the mapping $F_q:\dot{q} \mapsto F(q,\dot{q})$ defines a norm on the vector space $\tanb_q\funspace$, which makes $F$ a reversible Finslerian geometry \cite{Lang1999} and in turn ensures that $d'$ is a distance.
In this context, the existence of a geodesic between any pair of configurations is given by the well-known Hopf-Rinow theorem that we will detail in our case.
The study of the differential regularity of the operator is out of the scope of this article.

%Before stating the main theorem of this section, we need a second hypothesis.
For a path $\gamma$ in $\funspace$ the velocity field $v_{\gt, \gtdot}$ is not static in general.
If the velocity field is static, the path is exactly the flow induced by the ordinary differential equation (ODE) associated to this static field.
Moreover, for an element $q$ and an element $\dot{q}$ of the tangent $\tanb_q\funspace$, we can uniquely create a path $\gamma^{q,\qdot}$ solution of the ODE associated to the static velocity field $v_{q,\qdot}$ with the initial conditions given by $\gamma^{q,\qdot}_0=q$ and $\dot{\gamma}^{q,\qdot}_0=\qdot$. $\gamma^{q,\qdot}$ is the unique path passing through $q$ with derivative $\qdot$ which induce a static vector field equal to $v_{q,\qdot}$. The vector field $v_{q,\qdot}$ and the path $\gamma^{q,\qdot}$ are illustrated in Figure \ref{fig:diff_setting} (right).
%on some interval of the real line $I$ containing $0$. In particular for all $t$ in $I$:
%\begin{equation}
%    v_{\gamma^{q,\qdot}_t, \dot{\gamma}^{q,\qdot}_t} = v_{q, \qdot}\label{eq:static_ext}
%\end{equation}
\iffalse
More precisely, given some point $y_0$ in $\ambient$, the related differential equation:
\begin{equation}
    \dot{y} = v_{q, \dot{q}}(y) 
\end{equation}
admits a unique solution $y(t) = \phi^{q,\dot{q}}(t,y_0)$ such that $y(0)=y_0$.
The map $\phi^{q, \qdot}: \mathbb{R}\times \ambient \rightarrow \ambient$ is the flow associated with the vector field $v_{q,\qdot}$. It is known that for all $t$ the function $\phi^{q, \dot{q}}_t: x\mapsto \phi(t,x)$ is a diffeomorphism of $\ambient$ and, since $v_{q,\qdot}$ has null divergence, $\phi^{q, \dot{q}}_t$ is isochoric \cite{Arnold1989}.
By composition we can construct the path $\gamma^{q,\qdot}:t,x\mapsto \phi^{q,\qdot}(t,q(x))$ of isochoric diffeomorphisms and we immediately verify that $\gamma^{q,\qdot}_0=q$, $\dot{\gamma}^{q,\qdot}_0=\qdot$ and the static property as in Eq. \eqref{eq:static_ext}. The construction is illustrated in Figure \ref{fig:diff_illu} (right).
\fi
We now need a second hypothesis: these paths remain at least locally in $\funspace$:
\begin{equation*}
\text{\textbf{(H2)} : $\forall q, \dot{q}, \exists\epsilon>0,\forall t \in (-\epsilon,+\epsilon),\gamma^{q,\qdot}_t$ is an element of $\funspace$.}
\end{equation*}
\begin{theorem}\label{thm:norm_diff}
Under the hypotheses \textbf{(H1)} and \textbf{(H2)} the mapping:
\begin{equation}
F_q:\dot{q}\mapsto \int_{\partial B_q} \operatorname{max}(0, <v_{q, \dot{q}}(x),n_{q}(x)>) dA(x)\label{eq:norm}
\end{equation}
defines a norm on $\tanb_q \funspace$.
\end{theorem}
\begin{proof}
Let us consider some fixed $q$ in $\funspace$.
Using Stokes' theorem, the fact that $v_{q, \dot{q}}$ has a null divergence, that $max(0,C)$ is the positive part of a real $C$ and that $\partial B_q = q(\partial B)$ we obtain:
\begin{equation}
F_q(\dot{q}) = \frac{1}{2} \int_{q(\partial B)} |<v_{q,\dot{q}}(x),n_q(x)>| dA(x). \label{eq:avecvabs}   
\end{equation}
From this formulation it follows easily from the linearity of the integral and the inner product as well as from the norm properties of the absolute value that $F_q$ is positive $F_q(\dot{q})\geq 0$, homogeneous $F_q(\lambda\dot{q})=|\lambda|F_q(\dot{q})$  and sub-additive $F_q(\dot{q} + \dot{q}') \leq F_q(\dot{q}) + F_q(\dot{q}')$ for any $\dot{q},\dot{q}'$ in $\tanb_q\funspace$ and $\lambda$ in $\mathbb{R}$.

Now, let us consider $\qdot$ such that $F_q(\dot{q})=0$.
We use \textbf{(H2)} to construct the path $\gamma^{q,\qdot}$.
%:(-\epsilon,+\epsilon)\times\ambient\rightarrow \ambient$ in $\funspace$ that match $q$ with derivative $\qdot$ in $0$ and preserve the velocity vector field $v_{q,\qdot}$.
Since the integrand in Eq. \eqref{eq:norm} is always positive, $F_q(\dot{q})=0$ implies  that $<v_{q,\qdot}(y),n_q(y)>=0$ for every $y$ in $\partial B_q$
%, which mean $v_{q,\qdot}$ is tangent to $q(\partial B)$ everywhere.
Recall that $\partial B_q$ is a compact surface without boundary because it is isomorphic to $\partial B$. If a vector field is everywhere tangent to a compact surface without boundary, its integral curves passing by a point on this surface stay on the surface.
But $\gamma^{q,\qdot}$ is exactly the flow of $v_{q,\qdot}$ so we have that $\gamma^{q,\qdot}_t(y)$ is in $\partial B_q$ for every $y$ in $\partial B_q$ and $t$ in $(-\epsilon, +\epsilon)$.

We have $\gamma^{q,\qdot}_t(\partial B) \subset q(\partial B)$ and in turn $\mathcal{E}(\gamma^{q,\qdot}_t) = \mathcal{E}(q)
$ because $\partial B$ is a compact orientable surface without boundary
for all $t$ in $(-\epsilon, +\epsilon)$.
%, and elements of $\funspace$ are diffeomorphisms.
% It means that $\mathcal{E}(\gamma^{q,\qdot}_t) = \mathcal{E}(q)$
Using $\mathcal{E}$ injectivity from \textbf{(H1)} we have $\gamma^{q,\qdot}_t = q$ which makes $\gamma^{q,\qdot}$ a constant path on $(-\epsilon, +\epsilon)$
%. It implies that
and in turn $\qdot = \dot{\gamma}^{q,\qdot}_0 = 0$, which concludes the proof.
\end{proof}
\begin{remark}
Both \textbf{(H1)} and \textbf{(H2)} are necessary, and we can construct a counter example when either \textbf{(H1)} or \textbf{(H2)} is not satisfied.
\end{remark}
\iffalse
An example with a solid cylinder with a prismatic join and different configuration spaces of dimension 1 is given in Figure \ref{fig:counter_ex}. The configuration space is identified with the set of positions $(\theta, z)$ of a chosen reference point on the cylinder. On the left, the configuration space is $\{(\theta, 0), \theta \in \mathbb{R}\}$ the set of pure rotations, it verifies \textbf{(H2)} thanks to its group structure but \textbf{(H1)} is not verified because of the cylinder is a solid of revolution and any rotation does not sweep additional volume and we have $F_q(\qdot)=0$ with $\qdot$ not null. In the middle, the configuration space is $\{(\theta,\theta^2), \theta \in \mathbb{R}\}$ so \textbf{(H1)} is now verified because every configuration has a unique $z$ but due to the quadratic term \textbf{(H2)} is not verified and indeed in $\theta=0$ we can have $\dot{q} =(1,0)$ not null a pure instantaneous rotation so $F_q(\qdot)=0$. On the right the configuration space is $\{(\theta, \theta), \theta \in \mathbb{R}\}$ with both \textbf{(H1)} and \textbf{(H2)} verified so we have $F_q(\qdot)=0$ if and only if $\qdot=(0,0)$.
\begin{figure}[htp!]
  \centerline{%
    \includegraphics[width=0.3\textwidth]{multi_cylinder.png}
}
  \caption{\footnotesize \textbf{Left:} \textbf{(H2)} is verified but not \textbf{(H1)}; \textbf{Middle:} \textbf{(H1)} is verified but not \textbf{(H2)}; \textbf{Right:} both \textbf{(H1)} and \textbf{(H2)} are verified}
  \label{fig:counter_ex}
\end{figure}
\fi
When \textbf{(H1)} and \textbf{(H2)} and thus Theorem \ref{thm:norm_diff} holds, $F$ is a reversible Finslerian geometry
%\footnote{reversible Finslerian geometry is to Riemannian geometry what the norm is to the inner product}
and $d'$ as defined in Eq. \eqref{eq:dist_prime} is  automatically a distance on the configuration space.
To certify the existence of $d'$ geodesics between all pairs of configurations, we rely on topological properties of the space $\funspace$ equipped with $d'$ and the well-known Hopf-Rinow \cite{Bridson1999} theorem. In our case it can be written as follows:
\begin{theorem}[Hopf-Rinow]\label{thm:hopfrinow}
If \textbf{(H1)} and \textbf{(H2)} hold and $\funspace$ equipped with the distance $d'$ is locally compact, complete, and path connected, then any pair of configurations can be joined by a geodesic.
\end{theorem}
\iffalse
\begin{remark}
It is possible for $d'$ to be a distance even if $F$ is not a Finslerian geometry. However, it may be very difficult to ensure of the existence of geodesics which, as stated earlier and indicated in \cite{Lav06}, is important in robotics for interpolation. Moreover, \textbf{(H1)} and \textbf{(H2)} are not very restrictive and are generally satisfied in the case of poly articulated systems as detailed in the next subsection.
\end{remark}
\fi
\subsection{The case of a poly articulated system with rigid bodies}\label{sec:poly}
\textbf{Single rigid body.} When the system is a solid $B$ freely moving the diffeomorphisms are exactly the isometries and we can identify $\funspace$ with the space $SE(3)$ of solid placements. Given $[R,p]$ in $SE(3)$ where $R$ is a rotation matrix and $p$ a vector of $\ambient$, as well as $[\dot{R},\dot{p}]$ in $\tanb_{[R,p]}SE(3)$, we define the related diffeomorphism and its tangent element as $q(x) = Rx + p$ and $\qdot(x) = \dot{R}x + \dot{p}$.
%We have the following lemma for the expression of the local operator $F$.
\begin{lemma}\label{lm:simpF}
  Let $[\omega]_\times=R^T\dot{R}$ the skew-symmetric form of $\omega$ and $u=R^T\dot{p}$, where $\omega$ and $u$ are respectively the instantaneous rotation axis and the instantaneous linear velocity of the local frame.
  We have
  \begin{align}
F(q,\qdot)&=G_B(\omega,u)\label{eq:solid_op}\\
&=\frac{1}{2}\int_{\partial B} |\operatorname{det}(\omega,x,n(x))+<u,n(x)>|dS(x),
  \label{eq:chvar}
\end{align}
where $n(x)$ is the normal to $\partial B$ in $x$, and $dS(x)$ is the
element of surface on $\partial B$ at the same point.
\end{lemma}
\begin{proof}
We start from the formulation of $F$ in Eq. \eqref{eq:avecvabs}. Since $q(\partial B)$ is from $\partial B$ by the rigid transformation defined by $R$ and $p$, we have $n_q(y)=Rn(q^{-1}(y))$. Recall that $v_{q,\qdot} = \qdot \circ q^{-1}$. Using the change of variables $x = q^{-1}(y)$ with a Jacobian determinant equal to 1 because of the isochorism, we have:
\begin{equation}
F(q,\dot{q})=\frac{1}{2}\int_{\partial B} |<(\dot{R}x+\dot{t}),
Rn(x)>|dS(x).
\end{equation}
Using the adjoint $R^T$ of $R$ and the relation $<[\omega]_\times x,n> = \newline <\omega \times x,n> = \operatorname{det}(\omega,x,n)$ concludes the proof.
\end{proof}
In Eq. \eqref{eq:solid_op}, $(\omega, u)$ is a twist, an element of the spatial algebra as detailed by Featherstone in \cite{featherstone}. The integral $G_B$ is calculated over the fixed reference body $B$ which allows us with a simple, if tedious, calculation to find a closed form of $G_B$ for a polytope body $B$. See Appendix \ref{app:cf} for details.
%The details of this calculation are in the appendix \ref{app:cf}.

It is important to note that $F$ depends on $q$ and $\qdot$ only through the twist $(\omega, u)$ which comes from $[[\omega]_{\times}, u] = [R,p]^{-1}[\dot{R},\dot{p}]$ which is the Left lie algebra representation of $[\dot{R},\dot{p}]$ as detailed in \cite{Zefran1998}. This means that the local swept volume operator $F$ is left invariant, which in turn implies that the local swept volume does not depend on the position and orientation of the system but only on its velocity. The left invariance of $F$ allows us to use Theorems \ref{thm:norm_diff} and \ref{thm:hopfrinow} in the following corollary:
\begin{corollary}
If $B$ has no continuous symmetry group, the distance $d'$ is defined on $SE(3)$ and any pair of configurations can be joined by a geodesic.
\end{corollary}
\begin{proof}
For \textbf{(H1)} to hold with $\funspace=SE(3)$, it suffices by definition that the body $B$ does not have any continuous symmetry group. The symmetries of a compact are the group of solid transformations by which the image of the compact is the compact itself. If the functional space contains such a group, $\mathcal{E}$ is not injective. On the other hand, if $B$ has no symmetry group, $\mathcal{E}$ is necessarily injective on $SE(3)$.

%\footnote{The fact that $B$ does not have a continuous symmetry suffices since we only use injectivity locally in the proof of Theorem \ref{thm:norm_diff}.}.
Since $SE(3)$ is a Lie group, the hypothesis \textbf{(H2)} is always verified. On a Lie group, the path induced by the static velocity field are exactly the exponential map of the group \cite{Chirikjian2009}. If the configuration space is the whole group $SE(3)$, the exponential map output always belongs to $SE(3)$ and \textbf{(H2)} is verified.
%\footnote{The path extension in \textbf{(H2)} corresponds to the Lie exponential map in $SE(3)$ \cite{Eade2013}.}.
$F$ is left invariant \cite{Lang1999} and $SE(3)$ equipped with a norm is always locally compact and path connected. It is known that if a norm on a group is left invariant then it defines a complete topology \cite{Lang1999}. Using the Hopf-Rinow theorem \ref{thm:hopfrinow} concludes the proof.
\end{proof}
\textbf{Multiple rigid bodies.} When $B$ is the disjoint union of its $N$ rigid links $B_i$ we can associate the configuration space to a subset of $SE(3)^N$. If we assume first that the configuration space does not contain configuration for which the system is self colliding, $\mathcal{C}$ will be made of isochoric diffeomorphisms and the local swept volume is the sum of the volumes swept by each link:
\begin{align}
    F(q,\qdot) = \sum_{i=1}^N G_{B_i}\left(J_i(q)\qdot\right),\label{eq:fin}
\end{align}
where $J_i(q)$ is the Jacobian of the forward kinematic for the link $i$. More precisely, $J_i(q)\qdot$ is the vector of the twists which represents the instantaneous velocity of the link $i$ when the system is in configuration $q$ and varies according to $\qdot$, see \cite{featherstone} for details.
Analytic derivatives of the forward kinematic Jacobian are well known \cite{Carpentier2021} and we can easily compute the local operator $F$ for a poly-articulated system.
Moreover, note that if we interpret the forward kinematic as an immersion of the configuration space in $SE(3)^N$, $F$ is nothing else than the pullback of the norm on the Cartesian product of $N$ copies of $SE(3)$ respectively equipped with the norms $G_{B_i}$. The previous reasoning still holds even when the poly-articulated system may intersect itself, the only difference being that a collision region will be counted once for each link involved in the collision.
%Finally, we can mention the following corollary:
\begin{corollary}
For a poly-articulated system with $N$ rigid links such that none of the $B_i$ possesses a continuous symmetry group and the Jacobian of the forward kinematic always has full rank, the distance $d'$ is well defined on the configuration space and any pair of configurations can be joined by a geodesic.
\end{corollary}
\begin{proof}
If the Jacobian has full rank, the forward kinematic is an immersion of $\mathcal{C}$ in $SE(3)^N$. If none of the $B_i$  possesses a continuous symmetry group, each $G_{B_i}$ defines a complete Finslerian geometry on $SE(3)$. Using that the pullback by an immersion of the Cartesian product of a complete Finslerian geometry is a complete Finslerian geometry \cite{Bridson1999} concludes the proof.
\end{proof}

\section{Experiments}\label{sec:exp}
\begin{figure*}[htp]
\setlength\tabcolsep{1.2pt} % default value: 6pt
\renewcommand{\arraystretch}{0.2} % General space between rows (1 standard)
\begin{tabular}[b]{p{.01\textwidth}p{.18\textwidth}p{.01\textwidth}p{.18\textwidth}p{.01\textwidth}p{.18\textwidth}p{.01\textwidth}p{.18\textwidth}p{.01\textwidth}p{.18\textwidth}}
&\IncG[width=.17\textwidth,height=.17\textwidth]{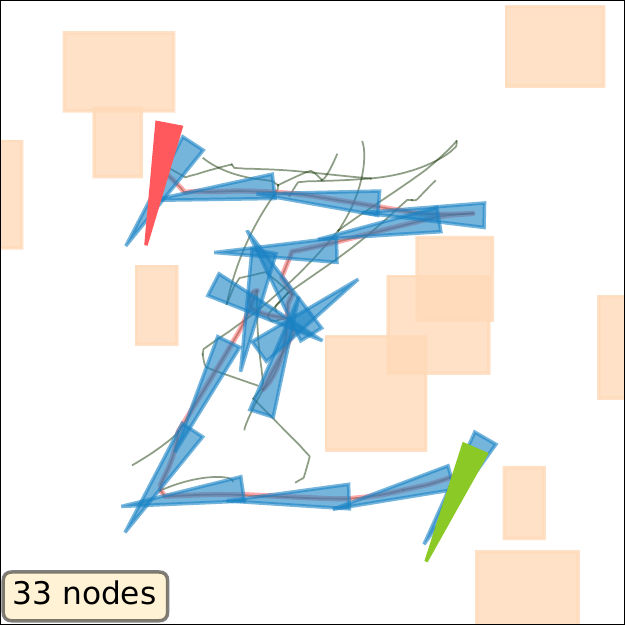}&
&\IncG[width=.17\textwidth,height=.17\textwidth]{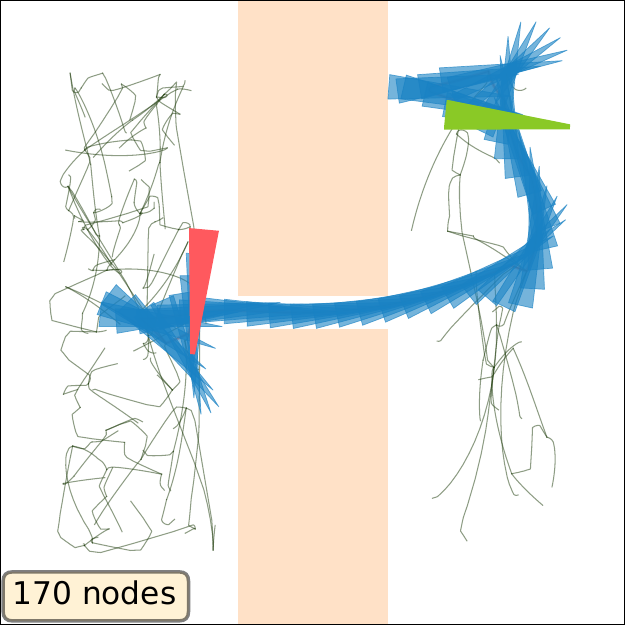}&
&\IncG[width=.17\textwidth,height=.17\textwidth]{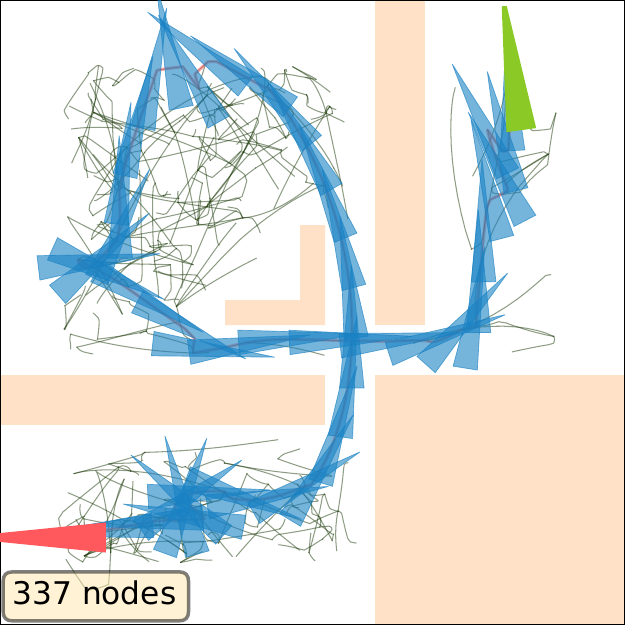}&
&\IncG[width=.17\textwidth,height=.17\textwidth]{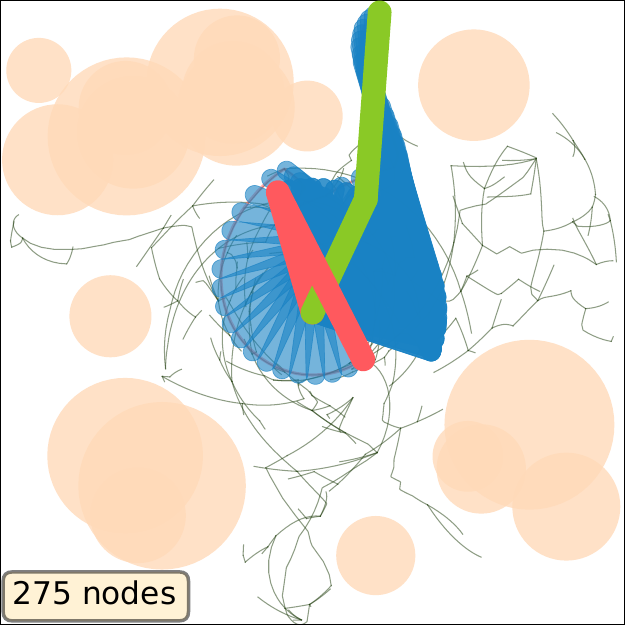}&
&\IncG[width=.17\textwidth,height=.17\textwidth]{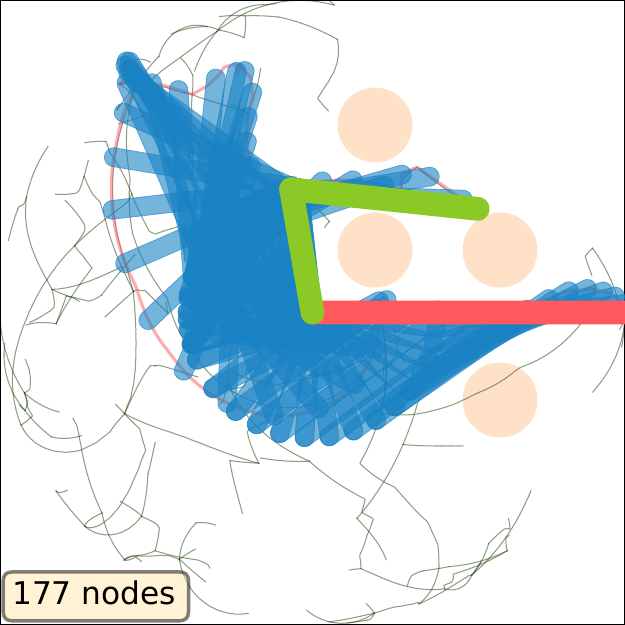}\\
\vspace{38.pt}a.&\IncG[width=.17\textwidth,height=.17\textwidth]{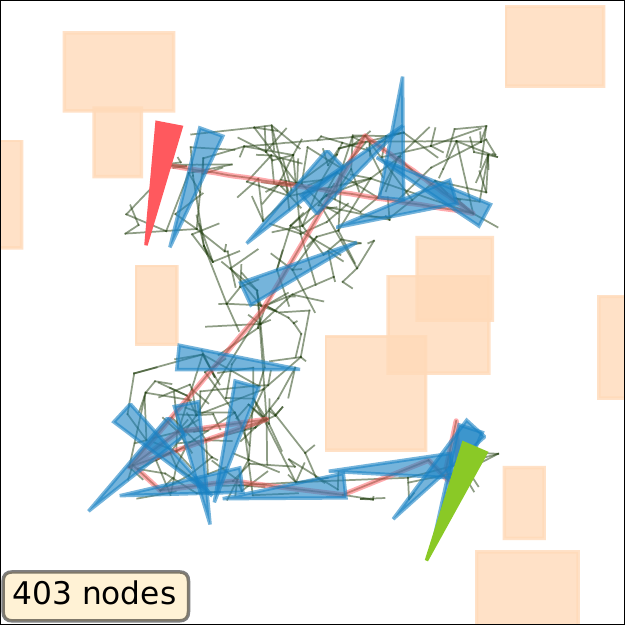}&
\vspace{38.pt}b.&\IncG[width=.17\textwidth,height=.17\textwidth]{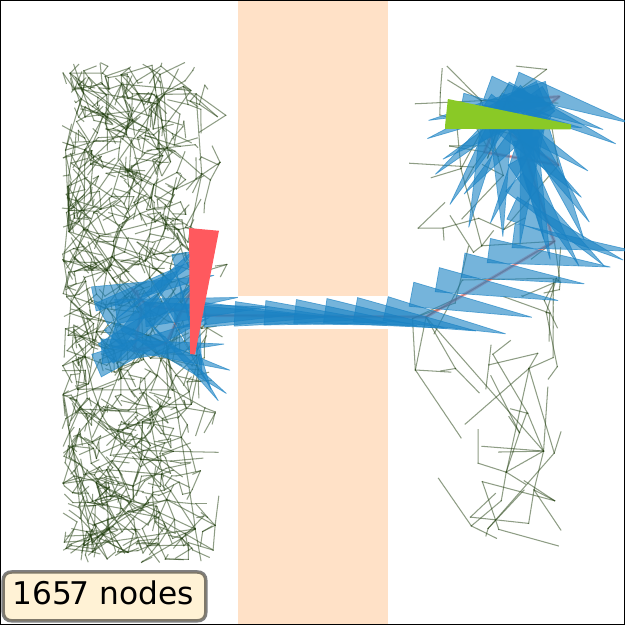}&
\vspace{38.pt}c.&\IncG[width=.17\textwidth,height=.17\textwidth]{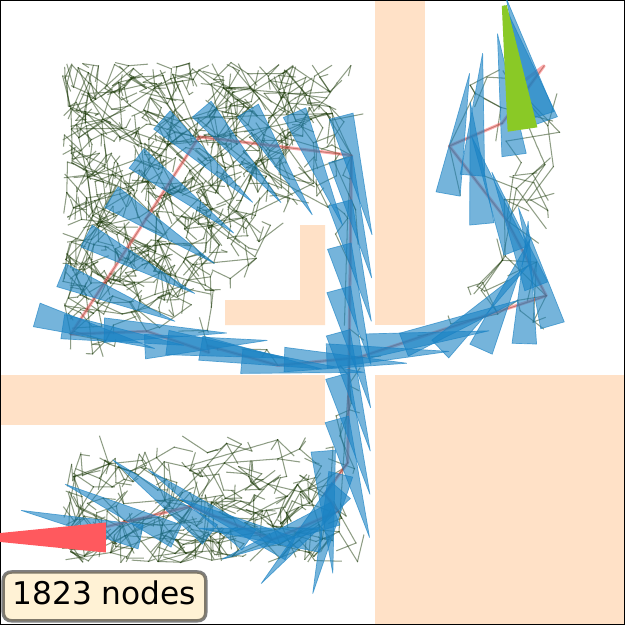}&
\vspace{38.pt}d.&\IncG[width=.17\textwidth,height=.17\textwidth]{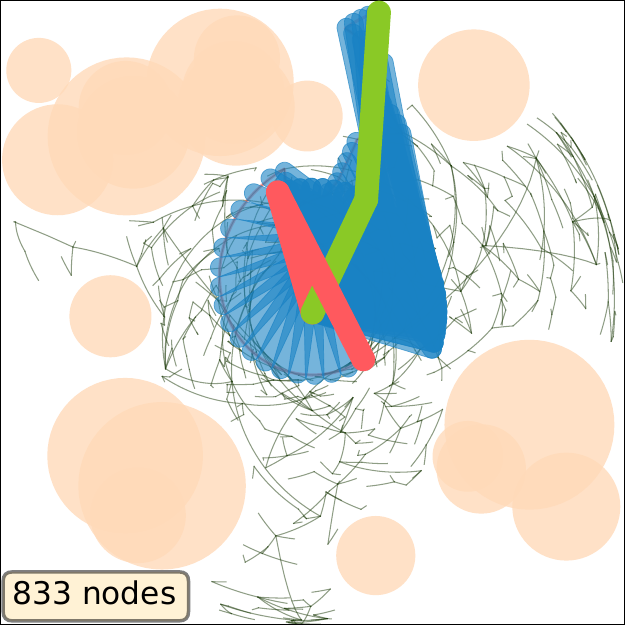}&
\vspace{38.pt}e.&\IncG[width=.17\textwidth,height=.17\textwidth]{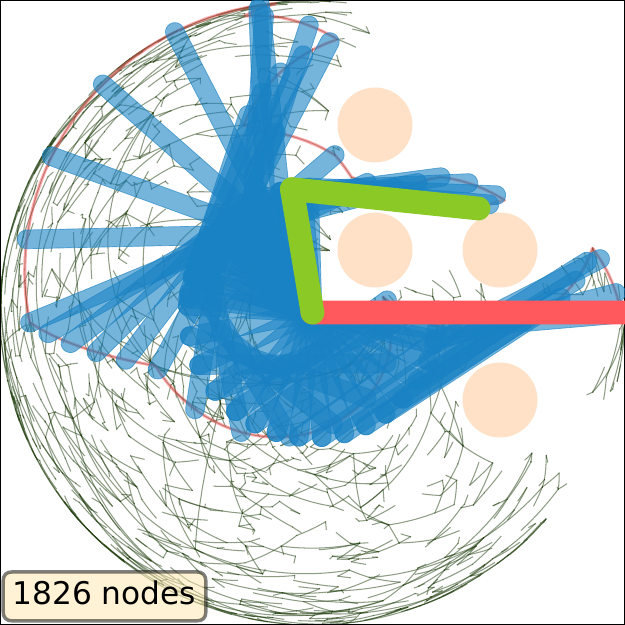}\\
\end{tabular}
  \caption{\footnotesize Different motion planning problems solved with RRT using the \MSV metric structure (top row) and the canonical Riemannian one (bottom row). From left to right: a. some random boxes, b. a narrow passage, c. a crossroad, d. a random cluttered environment and e. a narrow passage. The shape in red (resp green) is the starting (resp goal) configuration. In blue some snapshots of the final trajectory. We plot in dark green the position of the end effector for all the configurations in the tree.}
  \label{fig:motionplanning}
\vspace{-.5em}
\end{figure*}
\subsection{Technical details}
To illustrate our construction we restrict ourselves to the case of a 2D physical space with 2D solids in motion and planar poly-articulated arms.
The previous section provides a method for the computation of the distance and its geodesics in these cases.
First, we implement $G_B$ in the automatic differentiation framework CasADi \cite{Andersson2019}. Second, we use CasADi in conjunction with the rigid body library Pinocchio \cite{Carpentier2018} to obtain a differentiable implementation of the forward kinematic Jacobian and in turn a differentiable implementation of $F$ when using a smooth approximation of the absolute value as in \cite{Voronin2014}. In Finslerian geometry \cite{homogeneousfinsler}, geodesics are exactly the stationary trajectories for the energy functional associated with $F$ \cite{Lang1999}. Now we translate the infinite dimensional minimization problem into a finite dimensional one using $M$ time steps. For a pair of configurations $(a,b)$ in $\mathcal{C}^2$, we now call $q=(q_i)_{i=0..M}$ an array of $M+1$ elements of the configuration space $\mathcal{C}$ joining $a=q_0$ and $b=q_M$ as well as $w=(w_i)_{i=1..M}$ an array of $M$ elements of the configuration space lie algebra $\mathfrak{c}$ which represent the $M$ tangential vectors that maps $q_i$ to $q_{i+1}$ through the group exponential map \cite{Chirikjian2009}. We solve the multiple shooting problem:
\begin{align}
\begin{split}
&\min_{w \in \mathfrak{c}^{M}, q \in \mathcal{C}^{M+1}} \sum_{k=0}^{M-1} F(q_k, w_k)^2\\
&q_{k+1}=\operatorname{Exp}_{q_{k}}(w_k),\: k=0..M-1 ;\quad q_0=a, q_M=b,
\end{split}\label{eq:EM}
\end{align}
% &q_{k+1}=q_{k}\oplus w_k,\: k=0..M-1 ;\quad 
%\nonumber\\
%q_0=a, q_M=b\nonumber\label{eq:ocp}.
where $\operatorname{Exp}_{q}$ is the exponential map of the configuration space Lie group in configuration $q$ implemented through the Pinocchio library \cite{Carpentier2019}. We solve this minimization program using an off the shelf optimizer based on the interior point method \cite{Mao2018} because we can construct an initial guess that respects the constraints.
As this problem is not convex in general, we use multiple initial guesses which are constructed with the Riemannian geodesic and a random perturbation.
To compute $d'(a,b)$ it is sufficient to compute the sum of the terms $F$ taken on the steps of path solution to the problem \eqref{eq:EM}. 

\textbf{Note on complexity.}
We could use a Riemannian approximation of $F$ given by $g = \partial_{\qdot\qdot}F^2$ to obtain differential equations characterizing the geodesics \cite{Guigui2021}, but it would still require solving a bounded value problem with the same order of complexity.
The problem we solve is in general non-convex, and a global convergence study is out of the scope of this paper.
The overall complexity of the geodesics' computation is of the same order as that of trajectory optimization algorithms such as dynamic differential programming \cite{Schulman2014, Jallet21} but without the dynamical constraints because in our case the control variable is exactly the velocity. In motion planning, we use the geodesics as the paths associated to the tree edges \cite{Lav06} and the distance for the nearest neighbor searches. Thus, when we use the new metric, the computational complexity of the planning is of the same order as that of methods combining RRT and trajectory optimization such as \cite{Kamat2022}. However, in our case the cost of the trajectory optimization  problem is a distance, and we can combine the use of efficient spatial trees structure such as KD tree \cite{Bentley1975} with the optimization  procedure for every new node instead of a two stages approach as in \cite{Kamat2022}. Finally, even if dealing with obstacles is the main motivation behind the choice of the \MSVD as a metric structure, the distance and its geodesics construction do not depend on any specific environment choice. For a given system, we can first compute the solution of the optimization procedure for a large exploration of the free space and then use this computation to construct initial guesses for the geodesics and estimators for the distances when we explore specific environments. On the other hand, most trajectory optimization methods use the knowledge of the environment in the optimization procedure and no such pre-computation can be done.

In the remaining part of this section, we give a proof of concept of the \MSVD advantages over the canonical Riemannian distance by visualizing the geodesics of different systems and by observing the gain in terms of nodes needed in RRT. In the context of rigid body dynamic, the geodesic of the canonical Riemannian metric structure is the linear or spherical linear interpolation of each degree of freedom.
%\begin{itemize}
%\item 
% \item interpreting the geometry captured
%\item observing the gain in terms of nodes needed in a RRT
%\end{itemize}
\iffalse
\subsection{Geodesics, metric balls and invariance}
\begin{figure*}[t!]
  \vspace{-1em}
  \centerline{
  \includegraphics[width=.6\textwidth]{balls_rendered.png}
  }
  \vspace{-1em}
  \caption{\footnotesize \textbf{Left:} Illustration of the relative placement of a second configuration given the triplet $(x, y, \theta)$ 
\textbf{Middle:} Unit balls with radii from $1$ to $4$ in steps of $0.5$ one volume unit being the volume of the complete rod. \textbf{Right:} Views of the $r=4$ ball in $(\mathbf{x}, \mathbf{\theta})$, $(\mathbf{y}, \mathbf{\theta})$ and $(\mathbf{x}, \mathbf{y}$)}
  \label{fig:unitball}
\end{figure*}
\fi
\iffalse
\begin{table}[h]
\small
\begin{center}
\begin{tabular}{|c| c|c|} 
 \hline
 System & $\mathcal{C}$ & Triangular inequality \\ [0.5ex] 
 \hline\hline
 Rod & $SE(3)$ & $97.1\%$ \\ 
 \hline
 Arrow & $SE(3)$ & $98.2\%$ \\ 
 \hline
 Soft angle & $SE(3)$ & $89.7\%$ \\
 \hline
 64-side regular polygon & $SE(3)$ & $99.9\%$ \\
 \hline
 Double pendulum & $\mathbb{T}_2$ & $87.2\%$ \\
 \hline
 Triple pendulum & $\mathbb{T}_3$ & $78.1\%$ \\
 \hline
\end{tabular}
\caption{\small
Accuracy of the triangular inequality}\label{tab:acc}
\end{center}
\end{table}
\fi
\subsection{Geodesics}
We compute the geodesics of the minimum swept volume distance, and we compare them with the canonical Riemannian geodesics for different systems and different pairs of configurations. The trajectories obtained are illustrated in Figure \ref{fig:geodesic}.
We observe that the geodesics can have quite varied behaviors, but we often observe some characteristic sub paths of the motion that seems to be steady in some sense. This is an instance of the \textit{turnpike} property extensively studied by Trelat \cite{Trelat2015} in optimal control; the turnpike property of an optimal control problem is the fact that optimal trajectories decompose into transient phases and a steady state for most of the motion. A steady state is a solution of the static optimization  problem associated to an optimal control problem. In the context of rigid body motion and minimal swept volume metric structure, it ma correspond to regular motions such as pure translation or pure rotation. For the rod and the arrowhead shape the turnpike is a translation in the direction perpendicular to the shorter edge. For the  bracket shape it is a rotation with a center such that the edges are tangential to the circles induced by the rotation. For the arms the turnpike is the rotation of the base link when the remaining links are orthogonal to the base link.

\subsection{Motion planning}
Since the optimization problem \ref{eq:EM} is in general non-convex, it is important to check that the empirical solutions verify the axioms of a distance. To validate the use of the \MSVD in RRT  \cite{Lavalle00RRT} we verify the triangular inequality on 100000 triplets of random configurations. The accuracies are $97.1\%$ for the rod, $98.2\%$ for the arrowhead shape, $96.7\%$ for the  bracket shape, $91.2\%$ for the double pendulum and $87.2\%$ for the triple pendulum.

We compare the usage of the \MSV metric structure to the canonical Riemannian one in RRT when solving the motion planning problems illustrated in Figure \ref{fig:motionplanning}. We solve the problem of an arrowhead shape moving in a randomly generated set of boxes, a narrow passage, and a crossroad. We also solve the motion of a two-link arm moving in a randomly generated cluttered environment and in a set of tightened obstacles. For each scenario we select a pair of configurations and we run RRT until they are connected. We can observe in Figure \ref{fig:motionplanning} how RRT needs far fewer nodes to solve the problem when using the \MSV metric structure instead of the canonical Riemannian one.

For the case of the narrow passage Fig. \ref{fig:motionplanning} b., we plot in Figure \ref{fig:perf} the average number of nodes needed to solve the problem in function of the tunnel width. The advantage becomes more marked as the passage narrows. Let us consider two configurations of the arrowhead shape on either side of the narrow passage with their centers aligned with the tunnel. Based on the geodesics observed in Figure \ref{fig:geodesic}, the Riemannian geodesics will be this pure translation if in addition the orientation of the configurations are the same and aligned with the tunnel direction. On the other hand, the turnpike property of the \MSV geodesics will naturally produce this translation whatever the orientation of the configurations. Thus, the narrowness of the tunnel affects the requirements in position but also in terms of orientation in the Riemannian case. This explains why the advantage of the \MSV metric structure is even more important when the passage gets narrower.
\begin{figure}
  \centerline{
  \includegraphics[width=0.4\textwidth]{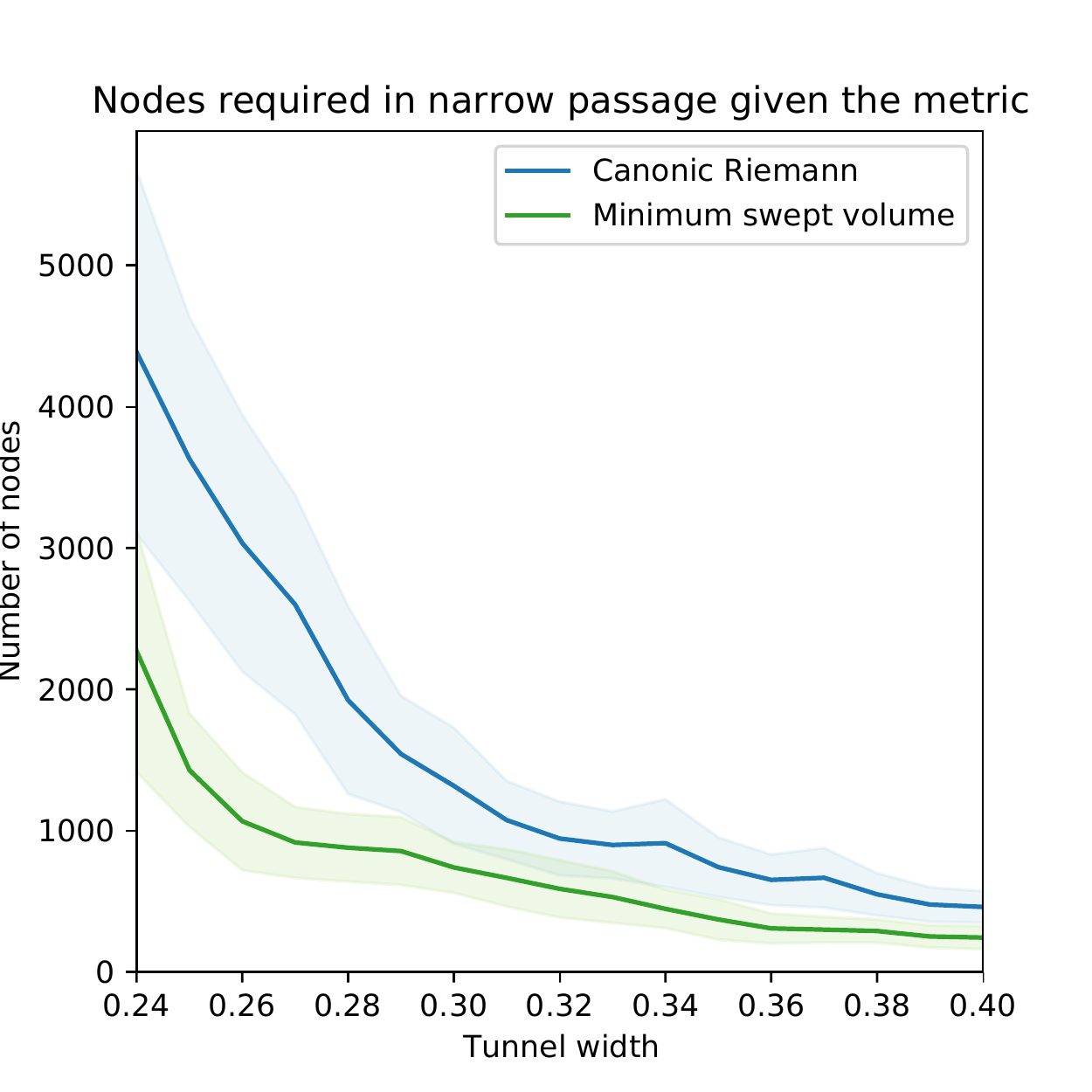}
  }
  \caption{Average number of nodes required to solve a narrow passage problem using RRT with the \MSV metric structure or the canonical Riemannian one over 50 random seeds.}
  \label{fig:perf}
\vspace{-1em}
\end{figure}

\section{Conclusion}\label{sec:ccl}
This paper lays the theoretical foundations of the \MSV metric structure on the configuration space. Our construction allows us to formulate the swept volume as an integral of local contributions so tools from Finslerian geometry ensure the existence of the geodesics and provide a computation procedure.

Our metric structure addresses some issues raised in the introduction because it is unique and does not rely on any parametrization choice. Our experiments do suggest it is appropriate for exploration in the presence of obstacles.
%\begin{itemize}
%\item unique and does not depend on any parameter settings,
%\item qualitatively appropriate for exploration in presence of obstacles.
%\end{itemize}
These two points makes the strength of our construction. The fact that it is useful for dense environment while being agnostic to the environment is a powerful feature.

There are several avenues for further work. First, we plan to implement a dedicated geodesic solver in C++ in a general 3D setting. We will then properly benchmark the total computational time of the approach. Second, the \textit{turnpike property} framework \cite{Trelat2015} seems to be a powerful tool to study the convergence of algorithms using the \MSV metric structure.
On the other hand, the Riemannian approximation $g = \partial_{\qdot, \qdot}F^2$ may also lead to other applications: it can be used to characterize the geodesics by differential equations, but it is also a simple way to create a new sampling method for the configuration space. We can indeed sample the configuration space following a law proportional to $\operatorname{det}(\partial_{\qdot,\qdot}F^2)$ the volume element of our geometry. This will favor tricky regions where displacements sweep more volume.

Geometric structures may have properties useful in robotic applications and we think that the formalism behind the \MSVD is a step in this direction.

\section*{APPENDIX}
\subsection{Closed form for a solid in motion}\label{app:cf}
The following calculation apply for a generic polyhedral body in a 3D physical space. Let us consider a polyhedral reference body $B$ with $I$ facets and each facet $f_i$ is a collection of $J_i$ edges represented by the pair of points $x_{i,j}^s,x_{i,j}^e$. We call $A_i$ the area of facet $i$ and $n_i$ the normal vector to the facet.

We have $G_B(\omega, u) = \frac{1}{2}\sum_{i=1}^I g_i(\omega,u)$ with:
\begin{equation*}
     g_i(\omega,u) = \begin{cases}
      A_i|<n_i,u>| & \omega \propto n_i\\
      \frac{1}{6}\sum_{k=1}^{J_i}\frac{[l_{i,j}^e|l_{i,j}^e| + l_{i,j}^s|l_{i,j}^s| + \chi(l_{i,j}^e, l_{i,j}^s)](z^s_{i,j} - z^e_{i,j})}{\|\omega\|^2 - <\omega,n_i>^2} & \omega \cancel{\propto} n_i,
    \end{cases}
\end{equation*}
where $l_{i,j}^\alpha = \operatorname{det}(\omega,x_{i,j}^\alpha,n_i) + <u,n_i>$ and $z_{i,j}^\alpha = <\omega,x_{i,j}^\alpha> - <\omega,n_i><x_{i,j}^\alpha,n_i>$ where $\chi(x,y) = xy\frac{|x| - |y|}{x - y}\text{ , }\chi(x,x) = x|x|$.

\begin{proof}
We decompose the integral into a sum of the integrals on each facet and we have $g_i(\omega, u) =\int_{f_i}|\operatorname{det}(\omega,x,n_i)+<u,n_i>|dS(x)$. If there is a real $\lambda$ such that $\omega = \lambda n_i$, the integrand does not depend on $x$ and we have the result.
Otherwise let us construct the direct orthonormal basis $(u_z, u_l, n_i)$ with $u_z = \frac{\omega-<\omega, n_i>n_i}{\sqrt{\|\omega\|^2-<\omega, n_i>^2}}$,
such that he plan span by $(u_z, u_l)$ contain  the facet and we have $dS = dl dz$ where $l$ and $z$ are the corresponding coordinate. With this basis we have
\begin{align*}
    det(\omega, x, n_i)= det(\omega-<\omega, n>n_i, x, n_i) =\sqrt{\|\omega\|^2-<\omega, n_i>^2}l.
\end{align*}
Applying the change of variable $l \mapsto \sqrt{\|\omega\|^2-<\omega, n>^2}l + <u,n>$ and $z \mapsto \sqrt{\|\omega\|^2-<\omega, n>^2}$ in the integral we can calculate:
\begin{align*}
    g_i(\omega,u) = \frac{1}{\|\omega\|^2-<\omega, n_i>^2}\int_{\phi(f_i)}|l|dldz,
\end{align*}
where $\phi(f_i)$ is the l-z transform of facet $f_i$, if we note the l-z value of a vertex $l^\alpha_j, z^\alpha_j$ where $\alpha$ stands for $e$ or $s$ we have the proper expression for $l^\alpha_j, z^\alpha_j$ in the frame.

We just need to calculate $|l|dldz$ over the polygon $\phi(F)$ to finish. Using the fact that the edges constitute a closed loop we have by integrating $l$:
\begin{align*}
    \int_{\phi(f_i)}|l|dldz = \frac{1}{2} \sum_{j=1}^L\int_{z^e_{j}}^{z^s_j}l_j(z)|l_j(z)|dz
\end{align*}

Using the affine parametrization of an edge $j$ we have $l_{i,j}(z) =\frac{z-z_{i,j}^s}{z_{i,j}^e-z_{i,j}^s}(l_{i,j}^e-l_{i,j}^s) + l_{i,j}^s$ and use it for a last change of variable and we have:
\begin{align*}
    \int_{\phi(f_i)}|l|dldz &=\frac{1}{6}\sum_{j}\frac{z_{i,j}^e - z_{i,j}^s}{l_{i,j}^e - l_{i,j}^s}(|l^s_{i,j}|^3 - |l^e_{i,j}|^3)\\
    &= \frac{1}{6}\sum_{j}(z_{i,j}^s - z_{i,j}^e)[l_{i,j}^s|l_{i,j}^s| + l_{i,j}^e|l_{i,j}^e| + \chi(l_{i,j}^e, l_{i,j}^s)]    
\end{align*}
Hence the result when we introduce $\chi$ such as $\chi(x,y) = xy\frac{|x| - |y|}{x - y}\text{ , }\chi(x,x) = x|x|$ for a factorization that also work when $l_{i,j}^e = l_{i,j}^s$.
\end{proof}

\section*{Acknowledgement}
This work was partially supported by the HPC resources from GENCI-IDRIS (Grant
011011181R1), the European Regional Development Fund
under the project IMPACT (reg. no. CZ.02.1.01/0.0/0.0/15
003/0000468)", the Inria-NYU collaboration agreement, the Louis Vuitton Louis Vuitton ENS Chair on Artificial Intelligence, and the French government under management
of Agence Nationale de la Recherche as part of the ”Investissements d’avenir” program, reference ANR-19-P3IA0001 (PRAIRIE 3IA Institute).

\bibliographystyle{IEEEtranS}
\bibliography{librairy}

\end{document}